\documentclass{article}
\usepackage{spconf,amsmath,graphicx}
\usepackage{times}
\usepackage{epsfig}
\usepackage{graphicx}
\usepackage{booktabs}

\usepackage{amsmath,amsthm,amssymb,amsfonts}

\usepackage{bm,graphicx,multirow,pifont}

\usepackage{array}
\newcolumntype{L}[1]{>{\raggedright\let\newline\\\arraybackslash\hspace{0pt}}m{#1}}
\newcolumntype{C}[1]{>{\centering\let\newline\\\arraybackslash\hspace{0pt}}m{#1}}
\newcolumntype{R}[1]{>{\raggedleft\let\newline\\\arraybackslash\hspace{0pt}}m{#1}}

\usepackage{subcaption}
\usepackage{xspace}
\usepackage{xcolor}
\usepackage{enumitem}
\usepackage{dsfont}
\usepackage{footmisc}
\usepackage{float}
\usepackage{pbox}
\usepackage{xcolor}

\makeatletter
\DeclareRobustCommand\bmvaOneDot{\futurelet\@let@token\bmv@onedotaux}
\def\bmv@onedotaux{\ifx\@let@token.\else.\null\fi\xspace}
\makeatother
\def\eg{\emph{e.g}\bmvaOneDot}

\def\ie{\emph{i.e}\bmvaOneDot}

\def\wrt{\emph{w.r.t}\bmvaOneDot}

\def\vs{\emph{vs}\bmvaOneDot}

\usepackage[pagebackref=true,breaklinks=true,colorlinks,bookmarks=false]{hyperref}

\renewcommand\vec[1]{\ensuremath\boldsymbol{#1}}
\renewcommand\cdots{...}

\newcommand{\tY}{\vec{\mathcal{Y}}}

\newcommand{\tX}{\vec{\mathcal{X}}}

\newcommand{\mX}{\mathbf{X}}

\newcommand{\mbr}[1]{\mathbb{R}^{#1}}

\newcommand{\vv}{\mathbf{v}}

\newcommand{\tE}{\vec{\mathcal{\lambda}}}
\newcommand{\tEH}{\vec{\mathcal{\hat{\lambda}}}}

\newcommand{\idx}[1]{\mathcal{I}_{#1}}
\newcommand{\semipd}[1]{\mathcal{S}_{+}^{#1}}

\newcommand{\vu}{\mathbf{u}}

\newcommand{\vphi}{\boldsymbol{\phi}}

\DeclareMathOperator*{\rank}{Rank}

\DeclareMathOperator*{\sym}{Sym}

\DeclareMathOperator*{\kronstack}{\uparrow\!\otimes}

\DeclareMathOperator*{\avg}{avg}
\DeclareMathOperator*{\sgn}{Sgn}
\DeclareMathOperator*{\hosvd}{HOSVD}

\newcommand{\expl}[1]{\text{e}^{#1}}

\newtheorem{theorem}{Theorem}

\newtheorem{proposition}{Proposition}
\newtheorem{remark}{Remark}

\newcommand{\mLambda}{\bm{\lambda}}
\newcommand{\mU}{\bm{U}}
\newcommand{\mV}{\bm{V}}

\newcommand{\vw}{\boldsymbol{w}}

\def\eg{\emph{e.g.}}

\newcommand{\tG}{\boldsymbol{\mathcal{G}}}

\newcommand{\mIdent}{\boldsymbol{\mathds{I}}}

\newcommand{\cE}{{\mathcal{\lambda}}}
\newcommand{\cEH}{{\mathcal{\hat{\lambda}}}}

\newcommand{\mPhi}{\boldsymbol{\Phi}}

\newcommand{\mM}{\boldsymbol{M}}

\newcommand{\mW}{\boldsymbol{W}}

\newcommand{\stkout}[1]{{\ifmmode\text{\sout{\ensuremath{#1}}}\else\sout{#1}\fi}}

\makeatletter
\def\ps@myheadings{%
    \let\@oddfoot\@empty\let\@evenfoot\@empty
    \def\@evenhead{\thepage\hfil\slshape\leftmark}%
    \def\@oddhead{{\slshape\rightmark}\hfil\thepage}%
    \let\@mkboth\@gobbletwo
    \let\sectionmark\@gobble
    \let\subsectionmark\@gobble
    }
  \if@titlepage
  \renewcommand\maketitle{\begin{titlepage}%
  \let\footnotesize\small
  \let\footnoterule\relax
  \let \footnote \thanks
  \null\vfil
  \vskip 60\p@
  \begin{center}%
    {\LARGE \@title \par}%
    \vskip 3em%
    {\large
     \lineskip .75em%
      \begin{tabular}[t]{c}%
        \@author
      \end{tabular}\par}%
      \vskip 1.5em%
    {\large \@date \par}
  \end{center}\par
  \@thanks
  \vfil\null
  \end{titlepage}%
  \setcounter{footnote}{0}%
}
\else
\renewcommand\maketitle{\par
  \begingroup
    \renewcommand\thefootnote{\@fnsymbol\c@footnote}%
    \def\@makefnmark{\rlap{\@textsuperscript{\normalfont\color{black}\@thefnmark}}}%
    \long\def\@makefntext##1{\parindent 1em\noindent
            \hb@xt@1.8em{%
                \hss\@textsuperscript{\normalfont\@thefnmark}}##1}%
    \if@twocolumn
      \ifnum \col@number=\@ne
        \@maketitle
      \else
        \twocolumn[\@maketitle]%
      \fi
    \else
      \newpage
      \global\@topnum\z@   
      \@maketitle
    \fi
    \thispagestyle{plain}\@thanks
  \endgroup
  \setcounter{footnote}{0}%
}
\makeatother

\makeatletter
\newcommand\fs@nobottomruled{\def\@fs@cfont{\bfseries}\let\@fs@capt\floatc@ruled
  \def\@fs@pre{}
  \def\@fs@post{}
  \def\@fs@mid{\kern2pt\hrule\kern2pt}%
  \let\@fs@iftopcapt\iftrue}
\makeatother

\usepackage[capitalize]{cleveref}
\crefname{section}{Sec.}{Secs.}
\Crefname{section}{Section}{Sections}
\Crefname{table}{Table}{Tables}
\crefname{table}{Tab.}{Tabs.}

\newcommand{\comment}[1]{}

\title{High-order Tensor Pooling with Attention for Action Recognition}

\name{Lei Wang\thanks{* This paper has been accepted for IEEE ICASSP 2024.}\textsuperscript{$*, \dagger,\S$} \qquad Ke Sun\textsuperscript{$\S,\dagger$} \qquad Piotr Koniusz\textsuperscript{$\S,\dagger$}
}
\address{$^{\dagger}$Australian National University, $^\S$Data61/CSIRO}

\begin{document}

\maketitle

\begin{abstract}
   We aim at capturing high-order statistics of feature vectors formed by a neural network, and propose end-to-end second- and higher-order pooling to form a tensor descriptor. Tensor descriptors require a robust similarity measure due to low numbers of aggregated vectors and the burstiness phenomenon, when a given feature appears more/less frequently than statistically expected. 
   The Heat Diffusion Process (HDP) on a graph Laplacian is closely related to the Eigenvalue Power Normalization (EPN) of the covariance/auto-correlation matrix, whose inverse forms a loopy graph Laplacian. 
   We show that the HDP and the EPN play the same role, i.e., to boost or dampen the magnitude of the eigenspectrum thus preventing the burstiness. We equip higher-order tensors with EPN which acts as a spectral detector of higher-order occurrences to prevent burstiness. We also prove that for a tensor of order r built from d dimensional feature descriptors, such a detector gives the likelihood if at least one higher-order occurrence is `projected' into one of binom(d,r) subspaces represented by the tensor; thus forming a tensor power normalization metric endowed with binom(d,r) such `detectors'. For experimental contributions, we apply several second- and higher-order pooling variants to action recognition, provide previously not presented comparisons of such pooling variants, and show state-of-the-art results on HMDB-51, YUP++ and MPII Cooking Activities. 
\end{abstract}

\begin{keywords}
high-order statistics, tensor descriptor, action recognition
\end{keywords}

\vspace{-0.5cm}
\section{Introduction}\label{sec::intro}

\ifdefined\arxiv
\newcommand{\PowH}{3.0cm}
\newcommand{\PowHB}{2.875cm}
\newcommand{\PowW}{3.65cm}
\else
\newcommand{\PowH}{2.8cm}
\newcommand{\PowHB}{3cm}
\newcommand{\PowW}{4.1cm}
\fi

Tensor descriptors require an appropriate aggregation / pooling mechanism to obtain robust estimates of covariance, auto-correlation or higher-order statistics and tackle the problem of \emph{burstiness}, a phenomenon that certain features may occur rarely or frequently in instances to classify but only a mere detection of presence/absence of a feature is relevant to the classification task. For instance, a CNN filter responding by a feature activation to a stimulus such as a tree leaf may respond once, few or many times across different spatial regions depending on whether a part or an entire tree is present in an image. However, reliably detecting a leaf (or few leaves), not the quantity per se, is a robust predictor of a tree \cite{me_tensor}. Furthermore, presenting a classifier with varying counts of a feature makes it simply harder for the classifier to generalize to unseen instances, \eg, if training images contained a small part of a tree (few leaves shown), it is likely that test images containing entire trees (thousands of leaves) will be misclassified as the decision boundary of a classifier is sensitive to the observed feature counts. For this reason, higher-order representations undergo a non-linearity such as Power Normalization (PN) \cite{me_tensor, kon_tpami2020a} which role is to reduce/boost contributions from frequent/infrequent visual stimuli in an image, respectively. Similar mechanisms apply to the temporal domain/action classification \cite{kon_tpami2020b}. PN methods can be split into element- and spectrum-wise, the latter referred to as Eigenvalue Power Normalization (EPN) \cite{me_tensor}.

So-called MaxExp and Gamma are two superior EPN operators with state-of-the-art performance \cite{me_tensor, kon_tpami2020a}. Thus, we perform a theoretical analysis of such EPN operators along with an application to action recognition. 


Our contributions are summarized as follows:
\begin{enumerate}[leftmargin=0.6cm]
%
\vspace{-0.2cm}
\item As MaxExp, Gamma and HDP can be thought of as pushforward functions acting on the discrete distribution representing an eigenspectrum, both MaxExp and Gamma are upper bounds of HDP, and MaxExp and Gamma can indeed be formulated as a modified system of Ordinary Differential Equations (ODE) defining the Heat Diffusion Equation~\cite{kon_tpami2020a}, and Gamma is  equal to HDP on the Log-Euclidean maps \cite{arsigny2006log}. We note that MaxExp, Gamma and HDP are approximations of the Grassmann metric~\cite{mehrtash_dict_manifold} which helps us design our spectral detector detailed next.

\vspace{-0.2cm}
\item EPN has been long speculated \cite{me_tensor} to perform spectral detection of higher-order occurrences. We prove that a tensor of order $r$, built from $d$ dimensional feature vectors, coupled with MaxExp indeed detects and yields the likelihood if at least one higher-order occurrence is `projected' into one of $\binom{d}{r}$  subspaces of dimension $r$ represented by the tensor; thus forming a Tensor Power Normalization metric endowed with $\binom{d}{r}$ such `detectors'.
\vspace{-0.2cm}
\item 
Additionally, 
we show how to backpropagate through the Higher Order Singular Value Decomposition (HOSVD) defined in \cite{lathauwer_hosvd,kolda_tensorrew}, which is essential in higher-order descriptors with EPN.
\end{enumerate}

As EPN prevents burstiness,  it replaces counting correlated features with detecting them, thus being invariant to their spatial/temporal extent.


%
\ifdefined\arxiv
\renewcommand{\PowH}{3.0cm}
\renewcommand{\PowHB}{2.875cm}
\renewcommand{\PowW}{3.65cm}
\else
\renewcommand{\PowH}{2.8cm}
\renewcommand{\PowHB}{3cm}
\renewcommand{\PowW}{4.1cm}
\fi

\section{Main Results}\label{sec:mainres}

\begin{figure*}[t]
\centering
\includegraphics[trim=0 0 0 0, clip=true, width=0.9\linewidth]{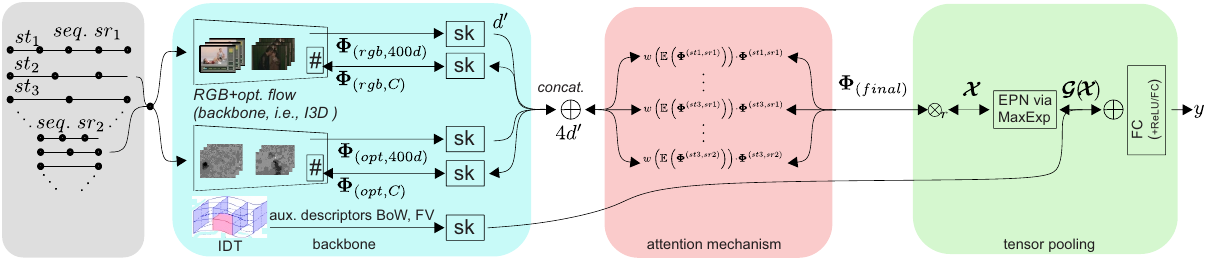}
\vspace{-0.3cm}
\caption{Our action recognition pipeline with the attention mechanism. See Section \ref{sec:pipe} for details.}
\vspace{-0.3cm}
\label{fig:pipe}
\end{figure*}

\subsection{HoT with EPN}\label{sec:epn-hot}

The EPN~\cite{me_tensor} performs a spectrum transformation on a given higher-order descriptor
$\tX\in\mbr{d_1\times d_2\cdots\times d_r}$ as follows
\vspace{-0.2cm}
\begin{align}
&{\left(\tE; \vec{U}_1,\cdots,\vec{U}_r\right)}=\hosvd(\tX),\label{eq:rawcod3}\\
&\tEH=g\left(\tE\right)\!,\label{eq:rawcod4}\\
&\tG(\tX)=((\tEH\times_{1}\!\vec{U}_1)\,\cdots)\times_{r}\!\vec{U}_r,\label{eq:rawcod5}
%
%
\end{align}
%

Below, we show that EPN in fact retrieves factors which quantify whether there is at least one datapoint $\vphi_n$, $n\!\in\!\idx{N}$, projected into each subspace spanned by $r$-tuples of eigenvectors from matrices $\vec{U}_1\!=\!\vec{U}_2 \cdots\!=\!\vec{U}_r$. For brevity, assume order $r\!=\!3$, a super-symmetric tensor,
and any $3$-tuple of eigenvectors $\vu$, $\vv$, and $\vw$ from $\vec{U}$. Note that $\vu\!\perp\!\vv, \vv\!\perp\!\vw$ and $\vu\!\perp\!\vw$ due to orthogonality of eigenvectors for super-symmetric tensors, \eg, $\mU\mLambda^\ddag\mV\!=[\tX_{:,:,1},\cdots,\tX_{:,:,d}]\in\mbr{d\times d^2}$ where $\mLambda^\ddag$ are eigenvalues of the unfolded tensor $\tX$ (note that they are not the core tensor of $\tX$ obtained by the HOSVD~\cite{lathauwer_hosvd,kolda_tensorrew}, which we denote as $\tE$). Moreover, note that if we have $d$ unique eigenvectors, we can enumerate $\binom{d}{r}$ $r$-tuples and thus $\binom{d}{r}$ subspaces $\mbr{d\times r}\!\subset\!\mbr{d\times d}$. For simplicity, we assume we have $d$ such vectors, that is $d\!=\!\rank(\tX_{:,:,i}), \forall i\!\in\idx{d}$, otherwise if $d'\!<\!d$, then we would have $\binom{d'}{r}$ subspaces instead.
 For brevity, let $||\vphi||_2\!=\!1$ and $\vphi\!\geq\!0$. Also, we write $\vphi_n$ instead of $\vphi$ for $n\!\in\!\idx{N}$ when index $n$ matters. Firstly, let us remove weights and zero-mean centering of our super-symmetric tensor descriptor:
\vspace{-0.4cm}
\begin{align}
& \tX\!=\!\frac{1}{N}\!\sum\limits_{n\in \idx{N}}{\kronstack}_r\,\vphi_n,
\label{eq:hod}
\end{align}
%
and write the `diagonalization' of $\tX$ by eigenvectors $\vu$, $\vv$, and $\vw$, which produces the core tensor with factors (spectrum of $\tX$):
\vspace{-0.3cm}
\begin{align}
& \cE_{\vu,\vv,\vw}=\tX\times_{1}\!\vu\,\times_{2}\!\vv\times_{3}\!\vw,
\label{eq:hoddiag}
\end{align}
%
where $\cE_{\vu,\vv,\vw}$ is a coefficient from the core tensor $\tE$ corresponding to eigenvectors $\vu$, $\vv$, and $\vw$. 
Now, let us combine Eq. \ref{eq:hod} and \ref{eq:hoddiag} which yields:
\vspace{-0.2cm}
\begin{align}
& \cE_{\vu,\vv,\vw}=\frac{1}{N}\!\sum\limits_{n\in \idx{N}}{\kronstack}_3\,\vphi_n\times_{1}\!\vu\!\times_{2}\!\vv\!\times_{3}\!\vw\nonumber\\
& \qquad\quad
=\frac{1}{N}\!\sum\limits_{n\in \idx{N}}\left<\vphi_n, \vu\right>\left<\vphi_n, \vv\right>\left<\vphi_n, \vw\right>.
\label{eq:hoddiag2}
\end{align}

\begin{theorem}\label{th:binom1}
Let $\vphi_n$ be `optimally' projected into subspace spanned by $\vu$, $\vv$ and $\vw$ when $\psi'\!_n\!=\!\left<\vphi_n, \vu\right>\left<\vphi_n, \vv\right>\left<\vphi_n, \vw\right>$ is maximized. As our $\vu$, $\vv$, and $\vw$ are orthogonal \wrt each other and $||\vphi_n||_2\!=\!1$, simple Lagrange Eq.  $\mathcal{L}\!=\!\Pi_{i\!=\!1}^r\vec{e}_i^T\!\vphi_n\!+\!\lambda(||\vphi_n||^2_2\!-\!1)$  yields maximum of $\kappa\!=\!(1/\sqrt{r})^r$ at $\vphi_n\!=\![(1/\sqrt{r}),\cdots,(1/\sqrt{r})]^T$. For each $n\!\in\!\idx{N}$, we store $\psi_n\!=\!\psi'\!_n/\kappa$ in a so-called event vector $\vec{\psi}$.

Assume that $\vec{\psi}\!\in\!\{0,1\}^{N}$ stores $N$ outcomes of drawing from Bernoulli distribution under the i.i.d. assumption for which the probability $p$ of an event $(\psi_n\!=\!1)$ and $1\!-\!p$ for $(\psi_n\!=\!0)$ are estimated by an expected value, $p\!=\!\avg_n\psi_n\!=\!\cE_{\vu,\vv,\vw}/\kappa$ ($0\!\leq\!\vec{\psi}\!\leq\!1$ in reality, $\cE_{\vu,\vv,\vw}, \kappa$ are introduced below in Eq. \eqref{eq:hoddiag2}). Then the probability of at least one projection event $(\psi_n\!=\!1)$ into the subspace spanned by $r$-tuples in $N$ trials is:
\vspace{-0.4cm}
\begin{equation}
\cEH_{\vu,\vv,\vw}\!=\!1\!-\!(1\!-\!p)^{N}=1\!-\!\left(1\!-\!\frac{\cE_{\vu,\vv,\vw}}{\kappa}\right)^{N}.
\label{eq:my_maxexp_tuples}
\end{equation}
\vspace{-0.4cm}
\end{theorem}
Each of $\binom{d}{r}$ subspaces spanned by $r$-tuples acts as a detector of projections into this subspace. Eq. \eqref{eq:my_maxexp_tuples} is the spectral MaxExp pooling with $\kappa$ normalization. We can also simulate a detector by defining $\lambda_{\vu,\vv}$ from Eq. \eqref{eq:hoddiag2} as the cosine distance explicitly, that is $\lambda_{\vu,\vv}\!=\!\kappa\text{cos}(\theta)\text{sin}(\theta),\,\kappa\!=\!2$. Note the detector-like responses of MaxExp, Gamma and HDP.

\begin{remark}
We note that the event vector $\vec{\psi}$ may contain negative values, so extending Eq. \eqref{eq:my_maxexp_tuples} to $\sgn(\cE_{\vu,\vv,\vw})\left(1\!-\!(1\!-\!\frac{|\cE_{\vu,\vv,\vw}|}{\kappa})^{\eta}\right)$, where $\eta\!\approx\!N$, divides each of $\binom{d}{r}$ subspaces into `positive' and `negative' parts, $\eta$ compensates for non-binary values in event vectors $\vec{\psi}\!\!$. However, the above extension has a non-smooth derivative. Thus, we use in practice the smooth SigmE operator  defined in our introduction as $2/\left(1\!+\!\expl{-\eta'\frac{|\cE_{\vu,\vv,\vw}|}{\kappa}}\right)\!-\!1\!\approx\!\sgn(\cE_{\vu,\vv,\vw})\left(1\!-\!(1\!-\!\frac{|\cE_{\vu,\vv,\vw}|}{\kappa})^{\eta}\right)$ (assume some correct choice of $\eta'$ given $\eta$). As third-order tensors are indefinite, SigmE is a good choice as it works with Krein spaces by design, is smooth and almost indistinguishable from MaxExp~\cite{kon_tpami2020a}.
\end{remark}

Finally, consider the dot-product between EPN-norm. tensors $\tG(\tX)$ and $\tG(\tY)$ computed by Eq. (\ref{eq:rawcod3}-\ref{eq:rawcod5}):
\vspace{-0.2cm}
\begin{align}
& \!\!\!\!\left<\tG(\tX),\tG(\tY)\right>\nonumber\\
&\!\!=\!\!\!\sum\limits_{\substack{\vu\in\mU(\tX)\\\vv\in\mV(\tX)\\\vw\in\mW(\tX)}}\sum\limits_{\substack{\vu'\!\in\mU(\tY)\\\vv'\!\in\mV(\tY)\\\vw'\!\in\mW(\tY)}}\!\!\!\!\!\!\!\cEH_{\vu,\vv,\vw}\cEH'_{\vu'\!,\vv'\!,\vw'}\,\left<\vu,\vu'\right>\left<\vv,\vv'\right>\left<\vw,\vw'\right>.\nonumber\\[-16pt]
&\label{eq:my_tens_subsp}
\end{align}
%
Hence, all subspaces of $\tX$ and $\tY$ spanned by $r$-tuples ($3$-tuples in this example) are compared against each other for alignment by the cosine distance. When two subspaces $[\vu,\,\vv,\,\vw]^T$ and $[\vu',\,\vv',\,\vw']^T$ are aligned well, for a strong similarity between these subspaces to be yielded, a detection of at least one $\vphi_n$ and $\vphi'_n$ in these subspaces evidenced by $\cEH_{\vu,\vv,\vw}$ and $\cEH'_{\vu'\!,\vv'\!,\vw'}$ is also needed. We call Eq. \eqref{eq:my_tens_subsp} together with Eq. (\ref{eq:rawcod3}-\ref{eq:rawcod5}) as Tensor Power Euclidean (TPE) dot-product with associated Tensor Power Euclidean metric $||\tX\!-\!\tY||_{\mathcal{G}}\!=\!||\tG(\tX)-\tG(\tY)||_F$.

\subsection{Backpropagating through HOSVD and/or SVD}\label{sec:hosvd-backprop}
\label{sec:backprop}

To backpropagate through HOSVD,  eigenvector matrices of  $\tX$, \eg, $\mU_1,\cdots,\mU_3$ for the third order, are given by $\mU_1\mLambda_1^\ddag\mV_1\!=\mM_1\!=\![\tX_{:,:,1},\cdots,\tX_{:,:,d}]\in\mbr{d\times d^2}$, $\mU_2\mLambda_2^\ddag\mV_2\!=\mM_2\!=\![\tX_{:,1,:},\cdots,\tX_{:,d,:}]\in\mbr{d\times d^2}$ and $\mU_3\mLambda_3^\ddag\mV_3\!=\mM_3\!=\![\tX_{1,:,:},\cdots,\tX_{d,:,:}]\in\mbr{d\times d^2}$. $\mM_1,\cdots,\mM_3$ are simply unfolded matrices of $\tX$ along the 1\textsuperscript{st}, 2\textsuperscript{nd}, and 3\textsuperscript{rd} mode.
\begin{proposition}
\label{pr:der}
Let $\mM^\#\!=\!\mM\mM^T\!=\!\mU\mLambda\mU^T$ be an SPD matrix with simple eigenvalues, \ie, $\lambda_{ii}\!\neq\!\lambda_{jj}, \forall i\!\neq\!j$. Then $\mU$ coincides also with the eigenvector matrix of tensor $\tX$ for the given unfolding. To compute the derivative of $\mU$ (we drop the index) \wrt $\mM$ (and thus $\tX$), one has to follow the chain rule:
\vspace{-0.3cm}
\begin{align}
& \frac{\partial \mU}{\partial M_{kl}}\!=\!\sum_{i,j}\frac{\partial \mU}{\partial (\mM\mM^T)_{ij}}\cdot\frac{\partial (\mM\mM^T)_{ij}}{\partial M_{kl}}, \;\;\nonumber\\
& \qquad\quad
\text{ where }\;\; \frac{\partial u_{ij}}{\partial\mM^\#}\!=\!u_{ij}(\lambda_{jj}\mIdent\!-\!\mM^\#)^{\dagger}.
\end{align}
\end{proposition}
\begin{proposition}
\label{prop:eig_chain}
For SVD, we simply have to backpropagate through the chain rule:
\vspace{-0.3cm}
\begin{align}
& \frac{\partial \mU\!\mLambda\mU^T}{\partial X_{m'n'}}\!=\!2\sym\left(\frac{\partial \mU}{\partial X_{m'n'}}\mLambda\mU^T\!\right)\!+\!\mU\!\frac{\partial\mLambda}{\partial X_{m'n'}}\mU^T\!\!,\,\quad
\nonumber\\
& \qquad\quad
\text{where}\quad\sym(\mX)\!=\!\frac{1}{2}(\mX\!+\!\mX^T\!).
\label{eq:eig_chain}
\end{align}
Let $\mX = \mU\mLambda\mU^T$ be an SPD matrix with simple eigenvalues, \ie, $\lambda_{ii}\!\neq\!\lambda_{jj}, \forall i\!\neq\!j$, and $\mU$ contain eigenvectors of matrix $\mX$, then one can apply 
$\frac{\partial\lambda_{ii}}{\partial X}\!=\!\vu_i\!\vu_i^T$ and 
$\frac{\partial u_{ij}}{\partial X}\!=\!u_{ij}(\lambda_{jj}\mIdent\!-\!X)^{\dagger}$.
\end{proposition}

\begin{proof}
The two rightmost steps $\frac{\partial\lambda_{ii}}{\partial X}$ and $\frac{\partial u_{ij}}{\partial X}$ are in \cite{magnus_der}.
\end{proof}

\begin{proposition}
For HOSVD, one has to follow the analogous chain rule as in Prop. \ref{prop:eig_chain}, but expanded from Eq. \eqref{eq:rawcod5}. To obtain the der. of $\tE$:  a  chain rule applies to Eq. \eqref{eq:hoddiag2}.
\end{proposition}

\section{Application to Action Recognition}\label{sec:pipe}

\noindent{\textbf{Action Recognition pipelines.}} 
Powerful CNN architectures include the two-stream network \cite{two_stream}, 3D spatio-temporal features  \cite{spattemp_filters}, spatio-temporal ResNet  \cite{spat_temp_resnet}, and the I3D network pre-trained on Kinetics-400 \cite{i3d_net}. However, these  networks operate on the RGB and optical flow frames thus failing to capture some domain-specific information which sophisticated low-level representations capture by design~\cite{lei_thesis_2017, lei_icip_2019, lei_tip_2019, wang2022uncertainty,wang2022temporal,qin2022fusing, lei2023, wang2023flow}. One prominent example are Improved Dense Trajectory (IDT) descriptors \cite{improved_traj} which are typically encoded with Bag-of-Words (BoW) of \cite{csurka04_bovw} or Fisher Vectors (FV) \cite{perronnin_fisherimpr} and fused with the majority of the modern CNN-based approaches of \cite{basura_rankpool2,hok,anoop_rankpool_nonlin,anoop_advers,iccv19_lei,acmmm21_lei, lei2023} at the classifier level for the best performance. Finally, attention in action recognition was also investigated in \cite{att_ar,wang20233mformer, lei2023}.

\noindent{\textbf{Improved Dense Trajectories.}} CNNs improve their performance if combined with IDT which involves several sophisticated steps: (i) camera motion estimation, (ii) modeling motion descriptors along motion trajectories captured with the optical flow, (iii) pruning inconsistent matches, (iv) removing focus from humans by the use of a human detector. IDT are usually combined with video descriptors such as Histogram of Oriented Gradients (HOG) \cite{hog2d,3D-HOG}, Histogram of Optical Flow (HOF) \cite{hof} and Motion Boundary Histogram (MBH)  \cite{improved_traj} which are complementary to each other. 
Thus, we follow others and encode IDT by BoW/FV. 
Below we evaluate our 
tensor descriptors combined with EPN (and SigmE which we call MaxExp in the reminder of the paper for simplicity as SigmE is an extension of MaxExp to Krein spaces). 
%


\noindent\textbf{Our pipeline overview.} Figure \ref{fig:pipe} introduces our action recognition pipeline. For each video sequence (an instance to classify), we first extract subsequences of lengths {\em $st_1,\cdots,st_3$} with the stride equal half of the subsequence length. This allows our approach to attain an invariance to action localization. Moreover, we apply various sampling rates, \eg, {\em $sr_1,sr_2$}, which brings some invariance to the action speed. We pass each subsequence through two I3D networks  \cite{i3d_net} which were trained on the RGB and optical flow videos, respectively. For each network, we extract a $400$ dimensional feature vector per subsequence whose coefficients coincide with the class labels of the Kinetics dataset (two-stream I3D CNN is pre-trained on Kinetics-400 off-the-shelf). Moreover, we use two more end-to-end trainable sub-streams of I3D denoted by ({\em \#}) whose outputs are two $C$-dim. feature vectors per subsequence (\eg, we  backpropagate \wrt the last FC and the last inception block of I3D). Thus, we obtain intermediate matrices with feature vectors denoted as $\mPhi_{(rgb,400d)},\mPhi_{(opt,400d)}, \mPhi_{(rgb,C)}, \text{ and }\mPhi_{(opt,C)}$ which we pass through the count sketching mechanism~\cite{cormode_sketch}, denoted ({\em sk}), to reduce the dimensionality to $d'$ so that concatenating the corresponding feature vectors with ({\em $\oplus$}) results in the size of $4d'\!\approx\!120$. 

\noindent\textbf{Attention mechanism.} 
At this stage, we pass each group of feature vectors $\vec{\Phi}^{(i,j)}$ by an attention mechanism, \ie,
%
%
%
$\vec{\Phi}^{(i,j)}_w\!=\!w\left(\mathbb{E}\left(\vec{\Phi}^{(i,j)}\right)\right)\!\cdot\!\vec{\Phi}^{(i,j)}$, 
where $i\!\in\!\{st_1,st_2,\cdots\}$ and $j\!\in\!\{sr_1,sr_2,\cdots,\}$ are selectors of the step size and the sample rate, respectively. Moreover, an attention network $w:\mbr{d'}\!\rightarrow\!\mbr{}$ takes a $d'$ dim. expected value over a given group of feature vectors $\mathbb{E}\left(\vec{\Phi}^{(i,j)}\right)$ as input to produce an attention score within range $[0;1]$. The attention network $w(\cdot)$ consists of an FC layer followed by a ReLU and another FC layer and a sigmoid function. The groups of feature vectors reweighted by such an attention mechanism form the final feature vector matrix $\mPhi_{(final)}\!\in\!\mbr{d\times N}$, where $d\!=\!4d'$, which is then simply passed via Eq. \eqref{eq:hod}. Finally, we obtain $\mX\!\in\!\semipd{d\times d}$  for $r\!=\!2$, or $\tX\!\in\!\mbr{d\times d\times d}$ for $r\!=\!3$, and we pass every $\tX$ via EPN~\cite{me_tensor} to obtain $\tG(\tX)\!\in\!\mbr{d\times d\times d}$, one per instance to classify. Finally, we combine our HoT descriptors with  auxiliary descriptors (\eg, IDT descriptors which are firstly count sketched to reduce their dimensionality, and then concatenated with HoT), as our pipeline shows in Fig.~\ref{fig:pipe}. 

\begin{table}[tbp!]
\vspace{-0.3cm}
\begin{center}
\resizebox{\linewidth}{!}{
\begin{tabular}{ l | c | c | c | c || l | c | c | c | c }
\toprule
$\qquad$SO+  & {\em sp1} & {\em sp2} & {\em sp3} & {\fontsize{9}{9}\selectfont mean} & $\qquad$TO+  & {\em sp1} & {\em sp2} & {\em sp3} & {\fontsize{9}{9}\selectfont mean} \\
\hline
(no EPN) 		& $76.2$ & $75.3$ & $76.7$ & $76.1$ & (no EPN) 			& $75.4$ & $74.0$ & $75.0$ & $74.8$\\
HDP    			& $81.4$ & $78.8$ & $80.1$ & $80.1$ & HDP 					& $81.8$ & $79.6$ & $81.3$ & $80.9$\\
MaxExp    		& $81.7$ & $79.1$ & $80.1$ & $80.3$ & MaxExp  		& $82.3$ & $79.9$ & $81.2$ & $81.1$\\
{\fontsize{9}{9}\selectfont MaxExp+IDT}   & $86.1$ & $85.2$ & $85.8$ & $\mathbf{85.7}$ & {\fontsize{9}{9}\selectfont MaxExp+IDT}   	& $87.4$ & $86.7$ & $87.5$ & $\mathbf{87.2}$\\
\hline
\end{tabular}
}
\resizebox{\linewidth}{!}{
\begin{tabular}{ c | c  |  c | c}
ADL+I3D~\cite{anoop_advers} $81.5$  &  Full-FT I3D~\cite{i3d_net} $81.3$ & SCK(SO+) +IDT~\cite{kon_tpami2020b} $85.1$ & SCK(TO+) +IDT~\cite{kon_tpami2020b} $86.1$  \\
\bottomrule
\end{tabular}
}
\end{center}
\vspace{-0.5cm}
\caption{({\em top}) Our model \vs ({\em bottom}) SOTA on HMDB-51.\label{tab:hmdb51f}}
\vspace{-0.3cm}
\end{table}

\begin{table}
\begin{center}
\resizebox{0.85\linewidth}{!}{
\begin{tabular}{ c | c | c | c || c | c}
\toprule
 & \multirow{2}{*}{\em static} & \multirow{2}{*}{\em dynamic} & \multirow{2}{*}{\em mixed} & mean & mean \\
&                &             &             & {\fontsize{8}{9}\selectfont stat/dyn}  & all \\
\hline
SO+MaxExp   				& $92.52$ & $82.03$ & $89.44$ & $87.3$ & $88.0$\kern-0.5em \\
SO+MaxExp+IDT   		& $94.92$ & $86.63$ & $96.02$ & $90.8$ & $92.5$\kern-0.5em \\
TO+MaxExp+IDT   		& $\mathbf{95.36}$ & $86.90$ & $\mathbf{97.04}$ & $91.1$ & $\mathbf{93.1}$\kern-0.5em \\
\hline
T-ResNet \cite{yuppp} & $92.41$ & $81.50$ & $89.00$ & $87.0$ & $87.6$\\
ADL I3D \cite{anoop_advers} & $95.10$ & $\mathbf{88.30}$ & - & $91.7$ & -\\
\bottomrule
\end{tabular}
}
\vspace{-0.2cm}
\caption{({\em top}) Our pipeline \vs ({\em bottom}) SOTA on YUP++.\label{tab:yupf}}
\end{center}
\vspace{-0.5cm}
\end{table}

\begin{table}
\begin{center}
\resizebox{0.48\textwidth}{!}{
\hspace{-0.3cm}
\begin{tabular}{ c | c | c | c | c | c | c | c | c }
\toprule
 & {\em sp1} & {\em sp2} & {\em sp3} & {\em sp4} & {\em sp5} & {\em sp6} & {\em sp7} & mAP \\
\hline
{SO+MaxExp+IDT}    		   & $75.7$ & $82.5$ & $79.4$ & $75.1$ & $75.7$ & $76.8$ & $75.9$ & $77.3$\\
{TO+MaxExp+IDT}      		 & $78.6$ & $83.4$ & $81.5$ & $78.8$ & $81.7$ & $79.2$ & $79.6$ & $\mathbf{80.4}$\\
\hline
\end{tabular}
}
\resizebox{0.48\textwidth}{!}{
\begin{tabular}{ c|c|c|c}
KRP-FS \cite{anoop_rankpool_nonlin} $70.0$  & KRP-FS+IDT \cite{anoop_rankpool_nonlin} $76.1$  & GRP \cite{anoop_generalized} $68.4$  & GRP+IDT \cite{anoop_generalized} $75.5$ \\
\bottomrule
\end{tabular}
}
\end{center}
\vspace{-0.5cm}
\caption{({\em top}) Our pipeline  \vs ({\em bottom})  SOTA on MPII.}

\label{tab:mpiif}
\vspace{-0.5cm}
\end{table}

\section{Experiments}\label{sec:exp}

%


\subsection{Datasets and Protocols}
\label{sec:data}

\vspace{0.05cm}
\noindent\textbf{HMDB-51} \cite{kuehne2011hmdb} 
consists of 6766 internet videos over 51 classes. 
We report the mean accuracy (\%) across three splits. 

\vspace{0.05cm}
\noindent\textbf{YUP++} \cite{yuppp}  contains so-called  video textures captured with the static or moving camera. It has 20 scene classes and 60 videos per class. We follow the standard evaluation protocols.

\vspace{0.05cm}
\noindent\textbf{MPII Cooking Activities} \cite{rohrbach2012database} has 64  activities in 3748 clips. 
We use the mean Average Precision (mAP) over 7-fold crossvalidation. For human-centric protocol \cite{anoop_rankpool_nonlin}, we use Faster-RCNN \cite{faster-rcnn} to crop video around humans.

\vspace{0.1cm}
\noindent\textbf{Settings.} 
For IDT, we applied PCA to trajectories (30 dim.), HOG (96 dim.), HOF (108 dim.), MBHx (96 dim.) and MBHy (96 dim.), and we obtained  213 dim. descriptors. We computed 1024 k-means and 256 dimensional GMM-based dictionaries, resp., and obtained 1024 and $\sim$110K final descriptors (FV was then count sketched to 4K). We used the Adam optimizer with $1\!e\!-\!4$ 
learning rate halved every 10 epochs. We ran our training for 50--200 epochs depending on the dataset. For second- and third-order descriptors, $d$ (we say $d$ but we mean $\!4d'$) was set within 200--1000 (400 for HMDB-51, 160 for YUP++, 520 for MPII) and $d\!\leq\!150$ (80 for HMDB-51, 60 for YUP++, 60 for MPII), respectively. On HMDB-51/YUP++, we used two FC layers intertwined by a ReLU followed by a Softmax classifier. On MPII, we used one FC (rightmost of Figure \ref{fig:pipe}). Subsequence lengths were set to 48, 64, 80, 96, and sampling rates to 1, 2 (and 3 for HMDB-51/YUP++).

\subsection{Evaluations\label{sec:eval}}

Table \ref{tab:hmdb51f} shows that the second- and third-order pooling ({\em SO}) and ({\em TO}) without our eigenvalue/spectral MaxExp underperform, with ({\em TO}) being worse than ({\em SO}), as without MaxExp, third-order descriptors are poorly estimated. The accuracy between MaxExp and HDP differs by $\sim\!0.25\%$ accuracy which validates our assumptions that EPN with MaxExp is in fact equivalent to the Heat Diffusion Process. Finally, third-order descriptors combined with MaxExp and IDT achieve $87.2\%$ accuracy and outperform ({\em SO}) by $1.5\%$. 

Table \ref{tab:yupf} shows that the accuracy on YUP++ obtained by third-order descriptors ({\em TO}) is $\sim\!0.6\%$ higher than the accuracy of ({\em SO}). Results on YUP++ are close to saturation thus we do not expect big gains, yet we outperform other methods in the literature, \eg, T-ResNet of \cite{yuppp} by $5.6\%$.

Table \ref{tab:mpiif} shows the results on  MPII. For MPII, we have applied human-centric crops and  used both IDT and I3D auxiliary features (I3D network was used with subsequences and global sequences subsampled to length 64). The third-order descriptor  scored $80.4\%$ mAP \vs $77.3\%$ attained by second-order pooling. We note that with $d\!=\!520$ and $d\!=\!60$ (we say $d$ but we mean $4d'$ here), second- and third-order descriptors contained 135K and 37820 unique coefficients, respectively. Decreasing $d$ for second-order descriptors did not help improve results which shows the benefit of third-order descriptors. Our best score gains $\sim\!4.9\%$ over the state-of-the-art GRP+IDT method.

\section{Conclusion}\label{sec:conc}

We have established a connection between Eigenvalue Power Normalization via the MaxExp operator with the Heat Diffusion Process. We have shown that MaxExp approximates HDP but does not require the Laplacian matrix. We have applied MaxExp to higher-order tensor descriptors and obtained Tensor Power Normalization metric endowed with $\binom{d}{r}$ `detectors' of the likelihood if at least one higher-order occurrence is `projected' into one of $\binom{d}{r}$  subspaces of dimension $r$ represented by the tensor. Our experiments on action recognition show that third-order descriptors outperform second-order.

{\small
\bibliographystyle{ieee}
\bibliography{egbib}
}

\end{document}